\documentclass[journal, twocolumn]{IEEEtran}
\ifCLASSINFOpdf
\else
  \usepackage[dvips]{graphicx}
\fi
\usepackage{url}
\usepackage{graphicx}
\usepackage{caption}
\usepackage{subcaption}
\usepackage[font=footnotesize]{caption}
\usepackage{pifont}
\newcommand{\cmark}{\ding{51}} 
\newcommand{\xmark}{\ding{55}} 

\usepackage{cite}
\usepackage{amsmath,amssymb,amsfonts,amsthm}
\usepackage{bbm}
\usepackage{algorithm2e}
\usepackage{algpseudocode}
\usepackage{mathtools}
\usepackage{braket}
\usepackage{bm}
\usepackage{multirow}
\usepackage{booktabs}
\usepackage{xcolor}
\usepackage{flushend}
\newtheorem{theorem}{Theorem}
\begin{document}

\title{Reliable Wireless Indoor Localization via Cross-Validated Prediction-Powered Calibration}

\author{Seonghoon Yoo, \IEEEmembership{Graduate Student Member, IEEE}, 
Houssem Sifaou, \IEEEmembership{Member, IEEE}, \\
Sangwoo Park, \IEEEmembership{Member, IEEE}, 
Joonhyuk Kang, \IEEEmembership{Member, IEEE}, 
and Osvaldo Simeone, \IEEEmembership{Fellow, IEEE}%
\thanks{Seonghoon Yoo and Joonhyuk Kang are with the School of Electrical Engineering, Korea Advanced Institute of Science and Technology, Daejeon 34141, South Korea (e-mail: shyoo902@kaist.ac.kr, jhkang@ee.kaist.ac.kr).}%
\thanks{Sangwoo Park is with the King’s Communications, Learning \& Information Processing (KCLIP) Lab, Centre for Intelligent Information Processing Systems (CIIPS), Department of Engineering, King’s College London, WC2R~2LS London, U.K. (e-mail: sangwoo.park@kcl.ac.uk).}
\thanks{H. Sifaou and O. Simeone are with the Institute for Intelligent Networked Systems, Northeastern University London, One Portsoken Street, London, E1 8PH, UK (email: o.simeone@northeastern.edu, h.sifaou@nulondon.ac.uk)}
\thanks{This work was supported by the Institute of Information \& communications Technology Planning \& Evaluation (IITP) grant funded by the Korea government (MSIT) (No. RS-2024-00444230, Development of Wireless Technology
for Integrated Sensing and Communication). This work was also supported by the Institute
of Information \& Communications Technology Planning \& Evaluation (IITP) under 6G $\cdot$ Cloud Research and Education Open Hub grant funded by the Korea government (MSIT) (IITP-2026-RS-2024-00428780); the work of O. Simeone was supported by the European Research Council (ERC) under the European Union’s Horizon Europe Programme (grant agreement No. 101198347), by an Open Fellowship of the EPSRC (EP/W024101/1), and by the EPSRC project (EP/X011852/1). \textit{(Corresponding author: Joonhyuk Kang and Osvaldo Simeone)}}
}

\markboth{IEEE XXX,~Vol.~XX, No.~XX, XXX~2026}%
{Yoo \MakeLowercase{\textit{et al.}}: Reliable Wireless Indoor Localization via Cross-Validated Prediction-Powered Calibration}

\maketitle

\begin{abstract}
Wireless indoor localization using predictive models with received signal strength information (RSSI) requires proper calibration for reliable position estimates. One remedy is to employ synthetic labels produced by a (generally different) predictive model. But fine-tuning an additional predictor, as well as estimating residual bias of the synthetic labels, demands additional data, aggravating calibration data scarcity in wireless environments. This letter proposes an approach that efficiently uses limited calibration data to simultaneously fine-tune a predictor and estimate the bias of synthetic labels, yielding prediction sets with rigorous coverage guarantees. Experiments on a fingerprinting dataset validate the effectiveness of the proposed method.
\end{abstract}


\section{Introduction}

\IEEEPARstart{W}{ireless} indoor localization has become a fundamental component of next-generation wireless networks, enabling a wide range of location-aware services such as navigation, tracking, and context-aware communications \cite{Yang2015Indoor_positioning, Koo2011CL_wifi_localize}. Recently, data-driven predictive localization methods have attracted significant attention for their adaptability and high positioning accuracy in complex propagation environments \cite{Zhu2024WCL_Self_supervised, Yan2021CL_ELM}. 

As illustrated in Fig.~\ref{fig_system}, for many location-aware services, it is essential not only to produce a nominal estimated position, but also to rigorously quantify the uncertainty of that estimate \cite{trevlakis2023localization,Hou2025TWC_PPI, Zheng2023TWC_confidence}. Distribution-free calibration methods such as conformal prediction \cite{angelopoulos2024theoretical} and \emph{risk-controlling prediction sets} (RCPSs) \cite{Bates2021RCPS} offer rigorous error guarantees. However, these methods rely on the availability of \emph{labeled calibration data points}, which may be scarce. Specifically, in the case of wireless localization, collecting labeled data generally requires expensive measurement campaigns.  

When \emph{unlabeled data} are available, one can construct a synthetic dataset with  \emph{pseudo-labels} generated by an existing predictive model. \emph{Prediction-powered inference} (PPI) \cite{Anastasios2023PPI} is a recently proposed framework for incorporating model-generated pseudo-labels from an auxiliary predictor,  while preserving statistical validity. While PPI applies to parameter estimation or learning, the work \cite{Romano2024RCPS_PPI} adapted PPI as a mechanism to construct prediction sets via RCPS  using both real and synthetic labels---an approach referred to henceforth as RCPS-PPI. 

\begin{figure}
\centering
\includegraphics[width=1\linewidth]{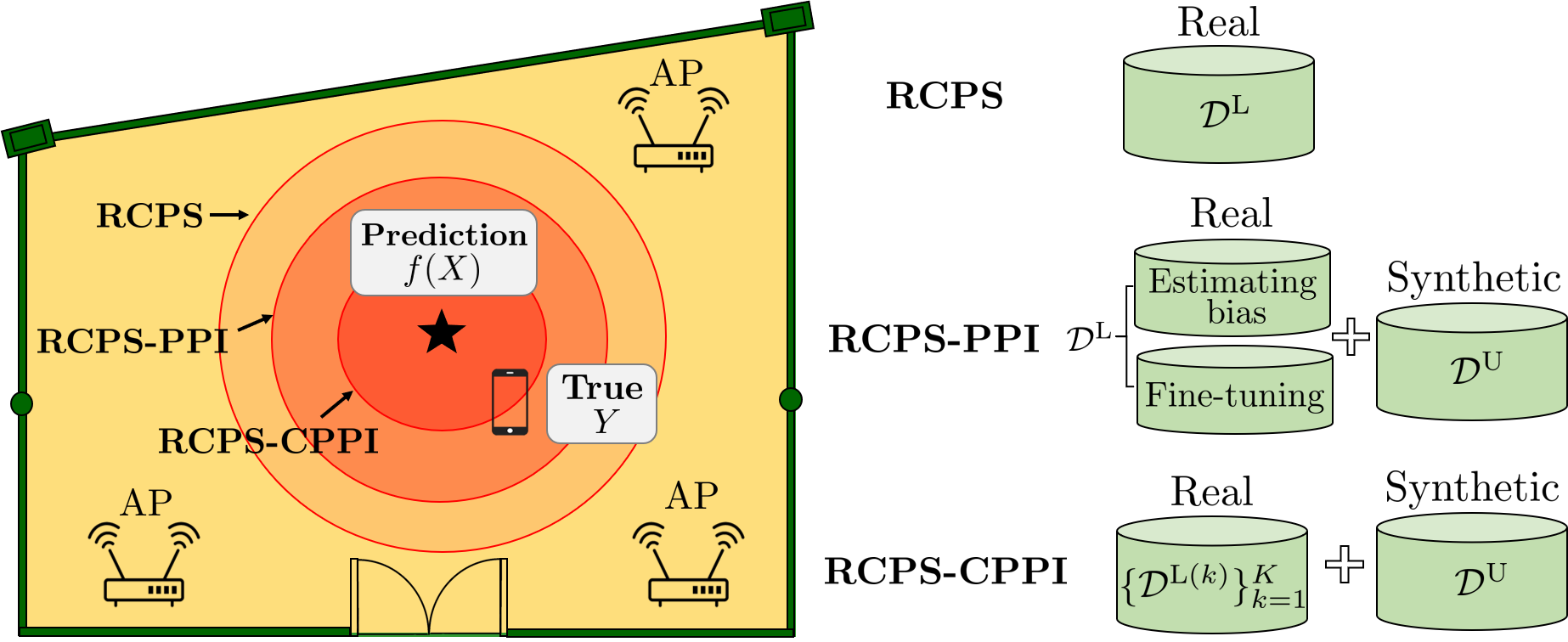}
\caption{Comparison of risk-controlling prediction sets (concentric circles) for the task of wireless indoor localization of mobile devices under three calibration strategies: RCPS \cite{Bates2021RCPS}, which uses only real-world labeled data; RCPS-PPI \cite{Romano2024RCPS_PPI}, which splits the real data into two subsets for fine-tuning the label-generating predictive model and for estimating the bias caused by synthetic labels; and the proposed RCPS-CPPI, which uses the entire labeled dataset for both predictor fine-tuning and bias estimation via cross-validation (see Table \ref{table}).}
\label{fig_system}
\end{figure}

\begin{table}[t]
\centering
\footnotesize
\caption{Comparison of state-of-the-art schemes.}
\label{table}
\renewcommand{\arraystretch}{1.1}
\setlength{\tabcolsep}{3.5pt}
\begin{tabular}{l|c|c|c}
{Scheme} & {Prediction / Learning} & {Synthetic data} & {Cross-validation} \\
\hline
RCPS \cite{Bates2021RCPS}           & Prediction & \xmark & \xmark \\
PPI \cite{Anastasios2023PPI}          & Learning   & \cmark & \cmark \\
RCPS-PPI \cite{Romano2024RCPS_PPI} & Prediction & \cmark & \xmark \\
CPPI \cite{Zrnic2024CPPI}, [16]  & Learning   & \cmark & \cmark \\
XB \cite{park2023few}             & Prediction & \xmark & \cmark \\
\textbf{RCPS-CPPI}                 & Prediction & \cmark & \cmark \\
\end{tabular}
\end{table}

In RCPS-PPI, a portion of the dataset must be set aside to train or fine-tune the auxiliary label-generating predictor on the given inference task, while the remaining portion of the labeled dataset is used for bias correction. Given limited labeled data, such splitting is inefficient and can degrade performance. 

\emph{Cross-PPI} (CPPI) \cite{Zrnic2024CPPI} was recently proposed to overcome this issue. CPPI uses $K$-fold cross-validation to utilize all labeled samples for both predictor training and calibration. However, while the use of CPPI is currently limited to learning, cross-validation was applied to set prediction in \cite{park2023few}. A summary of the state of the art can be found in Table \ref{table}.

In this letter, we develop a calibration scheme that enhances the efficiency of prediction sets by leveraging synthetic pseudo-labels. The proposed method, termed {RCPS-CPPI}, uses the available labeled data to simultaneously fine-tune a predictor and estimate the bias of the pseudo-labels, yielding prediction sets that satisfy a target risk level with a user-defined confidence probability. While RCPS can be applied to any inference task, we validate it on an indoor localization task \cite{dataset2014Torres}, demonstrating valid coverage with significantly reduced prediction set size compared to baseline approaches.

\section{Problem Definition}
\subsection{System Model}

As a concrete inference problem, consider a received signal strength indicator (RSSI) fingerprint-based indoor localization system, as illustrated in Fig.~\ref{fig_system}. In an indoor environment, multiple access points (APs) are deployed at fixed and known positions to periodically broadcast WiFi beacons, while a user equipment (UE) serves as the target node whose wireless signals are observed at the APs for localization. Fingerprint-based localization exploits the RSSI measurements collected from multiple APs to infer the position of the UE.

In practical deployments, RSSI fingerprints are significantly influenced by multipath propagation, shadowing, and device heterogeneity, resulting in complex and highly nonlinear relationships between signal strength and spatial position \cite{Zhang2023Indoor_TWC}. To cope with these challenges, model-based localization frameworks leverage labeled datasets of RSSI fingerprints to train a mapping between RSSI and UE locations.

\subsection{Risk-Controlling Prediction Sets}
We consider a general inference setting characterized by an input $X$, taking values in an arbitrary space, and an output $Y\in \mathcal{Y}$, where the domain $\mathcal{Y}$ may be discrete, for classification, or continuous, for regression. For the localization problem, the feature vector $X \in \mathbb{R}^m$ represents the RSSI readings from $m$ APs, whereas the target variable $Y \in \mathbb{R}^2$ denotes the two-dimensional user position, typically expressed in longitude and latitude coordinates. We are interested in quantifying the uncertainty of a pre-trained model $f(\cdot)$. 

Specifically, we aim at augmenting a decision $f(X)$ for any input $X$  with a prediction set $\Gamma_{\lambda}(X) \subseteq \mathcal{Y}$, depending on a threshold parameter $\lambda \in [0, \lambda_{\max}]$, that satisfies given statistical guarantees. The statistical performance of the set $\Gamma_{\lambda}(X) \subseteq \mathcal{Y}$ is measured by a loss function $\ell(Y,\Gamma_{\lambda}(X))$, such as the miscoverage loss 
\begin{equation}
    \label{miscoverage_loss}
    \ell(Y,\Gamma_{\lambda}(X))=\mathbbm{1}\{\,Y \notin \Gamma_{\lambda}(X)\,\},
\end{equation} where $\mathbbm{1}\{\cdot\}$ is the indicator function. In the example of Fig. 1, this loss provides a direct measure of the performance of the prediction set for localization by returning $\ell(Y,\Gamma_\lambda(X))=0$ if the prediction set $\Gamma_\lambda(X)$ includes the true UE location $Y$, while yielding $\ell(Y,\Gamma_\lambda(X))=1$ otherwise.

This way the expected risk
\begin{equation}
    R(\lambda) \;=\; \mathbb{E}_{P_{XY}}\big[\,\ell(Y,\Gamma_{\lambda}(X))\,\big]\,,
\end{equation}where the expectation is over the distribution $P_{XY}$ of the test data $(X,Y)$, 
measures the probability that the prediction set fails to include the true UE location. The target statistical requirement is that this expected risk is no larger than a  threshold $\alpha$ with probability at least $1-\delta$, i.e.,  \begin{equation}
    \Pr\!\big\{\,R(\hat{\lambda}) \le \alpha\,\big\} \ge 1-\delta\,,
    \label{rcps_def}
\end{equation} where probabilities $\alpha$ and $\delta$ are user-defined. As discussed in the next subsection, in \eqref{rcps_def}, the probability is taken with respect to the calibration data used to generate the prediction set $\Gamma_{\lambda}(X)$. If the condition \eqref{rcps_def} is satisfied, we say that the set $\Gamma_{\lambda}(X)$ is an $(\alpha,\delta)$-reliable prediction set.  

The general form of the prediction set $\Gamma_{\lambda}(X)$ as a function of the threshold $\lambda$ is given by \cite{angelopoulos2024theoretical} \begin{equation}
    \Gamma_{\lambda}(X) = \{\hat{Y}\in\mathcal{Y}: S(\hat{Y},f(X)) \le \lambda \}, 
    \label{eq:ball}
\end{equation}
where $S(\hat{Y},f(X))$ is an error score. By \eqref{eq:ball}, the prediction set includes all possible UE locations $\hat{Y}\in\mathcal{Y}$ whose error $S(\hat{Y},f(X))$ with respect to the prediction $f(X)$ is no larger than the threshold $\lambda$. For the example in Fig. \ref{fig_system}, which amounts to a multivariate regression problem, a typical choice for the score function is the Euclidean distance
\begin{equation}
S(\hat{Y},f(X))=\lVert \hat{Y}-f(X)\rVert_2
\label{euclidean_dis}
\end{equation}
between UE position $\hat{Y}$ and model prediction  $f(X)$. With this choice, the prediction set 
$\Gamma_{\lambda}(X)$ in \eqref{eq:ball} is a ball centered at the prediction   $f(X)$ with radius $\lambda$ as in Fig. \ref{fig_system}.

Following prior art \cite{Bates2021RCPS, Romano2024RCPS_PPI}, we make the following technical assumptions. The loss function $\ell(Y,\Gamma_{\lambda}(X))$ is bounded between 0 and 1; is non-increasing as the prediction set grows---i.e., for $\lambda' \leq \lambda$, and hence for $\Gamma_{\lambda}(X) \supseteq \Gamma_{\lambda'}(X)$, we have the inequality $
    \ell(Y,\Gamma_{\lambda}(X)) \le \ell(Y,\Gamma_{\lambda'}(X))\,$; and we have $\ell\big(Y,\Gamma_{\lambda_{\max}}(X)\big)=0$. These assumptions are satisfied by the miscoverage loss \eqref{miscoverage_loss}, as well as other standard losses, including the false negative rate in multi-label classification \cite{Bates2021RCPS}.

\subsection{Calibration Data}
As in \cite{Bates2021RCPS, Romano2024RCPS_PPI}, in order to determine the threshold $\lambda$ to be used in the prediction set $\Gamma_{\lambda}(X)$,  we assume the availability of a labeled dataset $\mathcal{D}^{\text{L}} = \{(X_i, Y_i)\}_{i=1}^{n}$ consisting of  $n$ i.i.d. samples drawn from the joint distribution $P_{XY}$. This is referred to as the \emph{labeled calibration dataset}.

Furthermore, as in \cite{ Romano2024RCPS_PPI}, we also assume that along with the labeled calibration dataset $\mathcal{D}^{\text{L}}$, we can access an \emph{unlabeled calibration dataset} $\mathcal{D}^{\text{U}} = \{\tilde{X}_j\}_{j=1}^{N}$ of $N$ i.i.d. input samples drawn from the marginal distribution $P_X$. The number of unlabeled calibration data points, $N$, is  considered to much larger than the number of labeled data points, $n$, i.e.,  $N \gg n$. Such unlabeled data naturally arises in wireless localization scenarios, where the dataset may include only RSSI measurements collected by participating users without sharing their position information \cite{Sifaou2025CPPI}.

\section{Background}
\subsection{Risk-Controlling Prediction Sets based on Real Data}
\label{Conventional}
To ensure the requirement  \eqref{rcps_def}, the RCPS approach \cite{Bates2021RCPS}  first constructs an upper confidence bound (UCB) $\hat{R}^{+}(\lambda)$ on the risk $R(\lambda)$ using the labeled calibration dataset $\mathcal{D}^\text{L}$. Then, it selects the smallest threshold $\lambda$  such that the UCB does not exceed the target risk level $\alpha$. 

Formally, let $\ell^{\text{L}}_{i}(\lambda) = \ell(Y_i,\Gamma_{\lambda}(X_i))$ be the loss on the $i$-th labeled sample for $i=1,\ldots,n$, so that the resulting empirical risk on the labeled data  is given by 
\begin{equation}
    \hat{R}^{\text{L}}(\lambda) = \frac{1}{n}\sum_{i=1}^{n} \ell^{\text{L}}_{i}(\lambda)\,.
    \label{Conventional_eq}
\end{equation} Using this empirical estimate, an UCB $\hat{R}^{+}(\lambda)$ can be obtained that satisfies the inequality 
\begin{equation}
    \Pr\!\big\{\,R(\lambda) \le \hat{R}^{+}(\lambda)\,\big\} \ge 1-\delta\,,
    \label{pointwise_UCB}
\end{equation}
where the probability is over the calibration dataset $\mathcal{D}^\text{L}$. The UCB can be constructed by leveraging the boundedness of the loss via methods such as the Waudby-Smith-Ramdas (WSR) estimator \cite{WSR2023WSR}. Finally, RCPS chooses the threshold as
\begin{equation}
    \hat{\lambda} \;=\; \inf\big\{\, \lambda : \hat{R}^+(\lambda) < \alpha\big\}\,, 
    \label{best_lambda}
\end{equation}
ensuring that the resulting set $\Gamma_{\hat{\lambda}}(X)$ is $(\alpha,\delta)$-reliable \cite{Bates2021RCPS}.

\subsection{PPI-based Risk-Controlling Prediction Sets}
\label{PPI}
Assume now access not only to the labeled calibration dataset $\mathcal{D}^\mathrm{L}$, but  also to the larger unlabeled dataset $\mathcal{D}^\mathrm{U}$. Assume also that we have an auxiliary parameterized predictor $g_{\theta}(X)$ providing estimates of label $Y$ for any given input $X$. 

The predictor $g_{\theta}(X)$ generally needs to be fine-tuned to provide accurate pseudo-labels on the given task. For example, the predictor $g_{\theta}(X)$ could be obtained from a foundation model pre-trained on a mixture of different datasets. For fine-tuning, RCPS-PPI \cite{Romano2024RCPS_PPI} uses part of the labeled data, $\mathcal{D}^{\mathrm{L}}_{\mathrm{ft}}$, reserving the rest of the labeled dataset, $\mathcal{D}^{\mathrm{L}}_{\mathrm{bc}}=\mathcal{D}^{\mathrm{L}}\setminus\mathcal{D}^{\mathrm{L}}_{\mathrm{ft}}$, for bias correction,  as explained next. 

For each unlabeled calibration input $\tilde{X}_j \in \mathcal{D}^{\text{U}}$, RCPS-PPI generates a pseudo-label $\hat{Y}_j = g_{\theta}(\tilde{X}_j)$ and evaluates the loss $\ell^{\text{U}}_{j}(\lambda) = \ell(g_{\theta}(\tilde{X}_j), \Gamma_{\lambda}(\tilde{X}_j))$ for $j=1,\ldots,N$. With these losses, one can estimate the expected risk via the empirical average  $\sum_{j=1}^{N} \ell^{\text{U}}_{j}(\lambda)/N$, but this estimate is generally biased.

To address this issue, RCPS-PPI introduces a bias correction term evaluated based on the labeled data. Specifically, for each $i$-th labeled data point $(X_i,Y_i)$ in dataset $\mathcal{D}^{\mathrm{L}}_{\mathrm{bc}}$, RCPS-PPI evaluates the difference \begin{equation} \Delta_i(\lambda) = \ell(g_{\theta}(X_i),\,\Gamma_{\lambda}(X_i))-\ell^{\text{L}}_{i}(\lambda) \end{equation} between the loss $\ell(g_{\theta}(X_i),\,\Gamma_{\lambda}(X_i))$ estimated using the prediction $g_{\theta}(X_i)$ and the true loss $\ell^{\text{L}}_{i}(\lambda)$. RCPS-PPI then constructs an estimator for the expected risk by subtracting from the unlabeled empirical loss the average bias correction obtained from labeled data:
\begin{equation}
    \hat{R}_{\text{PPI}}(\lambda) = \frac{1}{N}\sum_{j=1}^{N} \ell^{\text{U}}_{j}(\lambda)-\frac{1}{n_{\mathrm{bc}}}\sum_{i=1}^{n_{\mathrm{bc}}}\Delta_i(\lambda).
    \label{PPI_eq}
\end{equation}
The first term in \eqref{PPI_eq} is the empirical loss on unlabeled data, while the second term is the bias correction obtained using the $n_\mathrm{bc}$ labeled examples in set $\mathcal{D}^{\mathrm{L}}_{\mathrm{bc}}$. It can be shown that $\hat{R}_{\text{PPI}}(\lambda)$ is an unbiased estimator of the true risk $R(\lambda)$ \cite{Anastasios2023PPI}.

Similar to RCPS, RCPS-PPI obtains an UCB $\hat{R}^+_{\text{PPI}}(\lambda)$ using the unbiased estimate $\hat{R}_{\text{PPI}}(\lambda)$. The threshold $\hat{\lambda}$ is then evaluated as in \eqref{best_lambda}, ensuring RCPS-PPI is $(\alpha, \delta)$-reliable \cite{Romano2024RCPS_PPI}. 

\section{Prediction-Powered Calibration via Cross-Validation}
\label{sec:CPPI}
As explained in the previous section, RCPS-PPI uses part of the labeled data to train the labeling predictor $g_\theta(X)$. When the number of labeled data, $n$, is small, dedicating some of the data to this purpose may be problematic. CPPI \cite{Zrnic2024CPPI} addresses this issue via a $K$-fold cross-validation strategy for parameter estimation. In this section, we introduce an application of this principle to prediction set calibration.

\subsection{Cross-Validation-based Risk Estimate}
In the proposed RCPS-CPPI, the labeled calibration set $\mathcal{D}^{\text{L}}$ is partitioned into $K$ disjoint folds $\mathcal{D}^{\text{L}(1)}, \ldots, \mathcal{D}^{\text{L}(K)}$, each of size $n/K$. For each fold $k=1,\ldots,K$, we train a predictor $g_{\theta}^{(k)}(X)$ on the remaining $K-1$ folds, i.e., on dataset $\mathcal{D}^{\text{L}} \setminus \mathcal{D}^{\text{L}(k)}$. This ensures that all labeled data is used for training, yielding $K$ predictors $\{g_{\theta}^{(k)}(X)\}_{k=1}^K$, each learned on a subset comprising $(K-1)n/K$ labeled points.

Using the $K$ cross-validated predictors, we evaluate an unbiased estimate of the expected risk $R(\lambda)$ by obtaining $K$ unbiased estimates of the form (\ref{PPI_eq}), one for each predictor $g^{(k)}_{\theta}(X)$. In particular, for each fold $k$, we evaluate the loss $
    \ell^{\text{U}(k)}_{j}(\lambda) = \ell(g_{\theta}^{(k)}(\tilde{X}_j),\; \Gamma_{\lambda}(\tilde{X}_j))$ 
on each $j$-th data point $\tilde{X}_j$ in $\mathcal{D}^\text{U}$ using model $g_{\theta}^{(k)}(X)$. Furthermore, we compute a bias correction term $\Delta^{(k)}_{i}(\lambda)=\ell(g_{\theta}^{(k)}(X_i), \Gamma_{\lambda}(X_i))-\ell^{\text{L}}_{i}(\lambda)$ for each data point $(X_i,Y_i)$ in the fold $\mathcal{D}^{\text{L}(k)}$. Note that model $g_{\theta}^{(k)}(X)$ was trained on a dataset that excludes $(X_i,Y_i)$, ensuring that the loss $\ell(g_{\theta}^{(k)}(X_i), \Gamma_{\lambda}(X_i))$ is a valid unbiased estimate for the expected risk of the $k$-th predictor $g_{\theta}^{(k)}(X)$. Finally, RCPS-CPPI constructs the CPPI risk estimate \cite{Zrnic2024CPPI}
\begin{equation}
\hat{R}_{\text{CPPI}}(\lambda) = \frac{1}{K}\sum_{k=1}^{K}\bigg(\frac{1}{N}\sum_{j=1}^{N} \ell^{\text{U}(k)}_{j}(\lambda) - \frac{1}{n}\sum_{i \in \mathcal{D}^{\text{L}(k)}} \Delta^{(k)}_{i}(\lambda)\bigg). 
    \label{CPPI_eq}
\end{equation}
In (\ref{CPPI_eq}), the term $\Delta^{(k)}_{i}(\lambda)$ adjusts for the bias of model $g_{\theta}^{(k)}(X)$. By design, the quantity $\hat{R}_{\text{CPPI}}(\lambda)$ is an unbiased estimator of $R(\lambda)$, and it uses all available labeled samples both in forming the predictors and in correcting bias.
\subsection{RCPS-CPPI}
To obtain an UCB from the CPPI risk estimator \eqref{CPPI_eq}, we rewrite (\ref{CPPI_eq}) as an empirical average of unbiased estimates $\hat{R}^{(k)}_\text{CPPI}(\lambda)$, each obtained by using labeled data from a different fold $\mathcal{D}^{\text{L}(k)}$. To this end, for each fold $k$, we define the term $\hat{R}_\text{CPPI}^{(k)}(\lambda)$ by including the subset of terms in \eqref{CPPI_eq} associated with that fold as
\begin{equation}
    \hat{R}^{(k)}_{\text{CPPI}}(\lambda) = \frac{1}{n_k}\sum_{i=1}^{n_k} \Bigg[\frac{n_k}{N}\sum_{j=(i-1)\frac{N}{n_k}+1}^{\,i\frac{N}{n_k}} \ell^{\text{U}(k)}_{j}(\lambda)-\Delta_{i}^{(k)}(\lambda) \Bigg],
    \label{CPPI_eq2}
\end{equation}
where $n_k = n/K$ is the size of dataset $\mathcal{D}^{\text{L}(k)}$.

Since $\hat{R}_{\text{CPPI}}^{(k)}(\lambda)$ is an unbiased estimator of the expected risk $R(\lambda)$, an UCB ${R^{+\,(k)}_\text{CPPI}}(\lambda)$ can be evaluated using methods, e.g., the WSR estimation \cite{WSR2023WSR}. Specifically, RCPS-CPPI determines an UCB satisfying the inequality $\Pr\{R(\lambda) \le R^{+\,(k)}_{\text{CPPI}}(\lambda)\} \ge 1 - {\delta}/{K}$ for each fold $k$. Finally, RCPS-CPPI evaluates
\begin{equation}
    \hat{R}^{+}_{\text{CPPI}}(\lambda) = \min_{1 \le k \le K}\; R^{+\,(k)}_{\text{CPPI}}(\lambda)\,
    \label{CPPI_eq4}
\end{equation}
and applies selection rule \eqref{best_lambda} using the estimate $\hat{R}^{+}_{\text{CPPI}}(\lambda)$ instead of $\hat{R}^+(\lambda)$. Using grid search over parameter $\lambda$ with $m$ points, the complexity of this operation is of order $\mathcal{O}(K(N+n)m)$.
\subsection{Theoretical Guarantees}
The following theorem formalizes the coverage guarantee of RCPS-CPPI.
\begin{theorem}
RCPS-CPPI produces an $(\alpha,\delta)$-reliable prediction set.
\end{theorem}
\begin{proof}
    Let $\mathcal{E}_k = \big \{R(\lambda) \leq R^{+\,(k)}_{\textnormal{CPPI}}(\lambda)\}$. By the union bound over the $K$ events $\{\mathcal{E}_k\}_{k=1}^{K}$, we obtain 
        \begin{equation}
            \textnormal{Pr}\bigg\{\bigcup_{k=1}^{K}\mathcal{E}_k^{c}\bigg\}\leq \sum_{k=1}^{K}\textnormal{Pr}\big\{\mathcal{E}_k^{c}\big\} \leq K\cdot \frac{\delta}{K}=\delta,
        \end{equation}
        which implies
\begin{equation}
    \textnormal{Pr}\bigg\{\,\bigcap_{k=1}^{K}\mathcal{E}_k\,\bigg\}=\textnormal{Pr}\bigg\{\,R(\lambda)\leq\min_{1 \le k \le K}\; R^{+\,(k)}_{\text{CPPI}}(\lambda)\,\bigg\}\geq 1-\delta.
\end{equation}
Therefore, the RCPS-CPPI is $(\alpha,\delta)$-reliable.
\end{proof}



\section{Experiments}
\label{Experiments}
\subsection{Setup}
We evaluate the proposed approach on a wireless indoor localization task using a WiFi fingerprinting dataset \cite{dataset2014Torres, Cheng2022FL_WCL}. We randomly sample a subset of 100 labeled examples to train the base model $f$ and simulate scenarios with limited calibration datasets $\mathcal{D}^\text{L}$, with $n$ varying from 50 up to 500 labeled calibration points. The remaining data are used as an unlabeled calibration dataset $\mathcal{D}^{\text{U}}$. We reserve 30 labeled samples for a test set to evaluate coverage and set size \cite{yoo2025ml_wcp}. 

Following the setup in \cite{Sifaou2025CPPI}, we adopt an extreme learning machine (ELM) regressor as the base model $f(X)$ to be calibrated. We adopt the Euclidean-distance score (\ref{euclidean_dis}) and the miscoverage loss (\ref{miscoverage_loss}). The auxiliary predictor $g_{\theta}(X)$ is implemented as a fully-connected neural network with three hidden layers. We set the target risk level to $\alpha = 0.1$ and confidence to $1-\delta = 0.9$ for all calibration methods. We consider $K=5$ folds for the CPPI method by default.

\subsection{Results}
We report the empirical coverage, i.e., the fraction of test points whose true location lies inside the prediction set, and the \emph{inefficiency}, defined as the average radius of the prediction sets $\Gamma_{\hat{\lambda}}(X)$, on the test samples. A larger inefficiency implies a more significant uncertainty on the position of the UE, which in turn affects the performance of location-based functionalities such as beamforming and resource allocation \cite{trevlakis2023localization}. Apart from RCPS (Section~\ref{Conventional}), and RCPS-PPI (Section~\ref{PPI}), we also consider a baseline semi-supervised (SS) scheme that uses both labeled and unlabeled data without any bias correction and directly applies RCPS on the combined data.

\begin{figure}[t]
   \hspace*{-6pt}
   \includegraphics[width=1.15\linewidth]{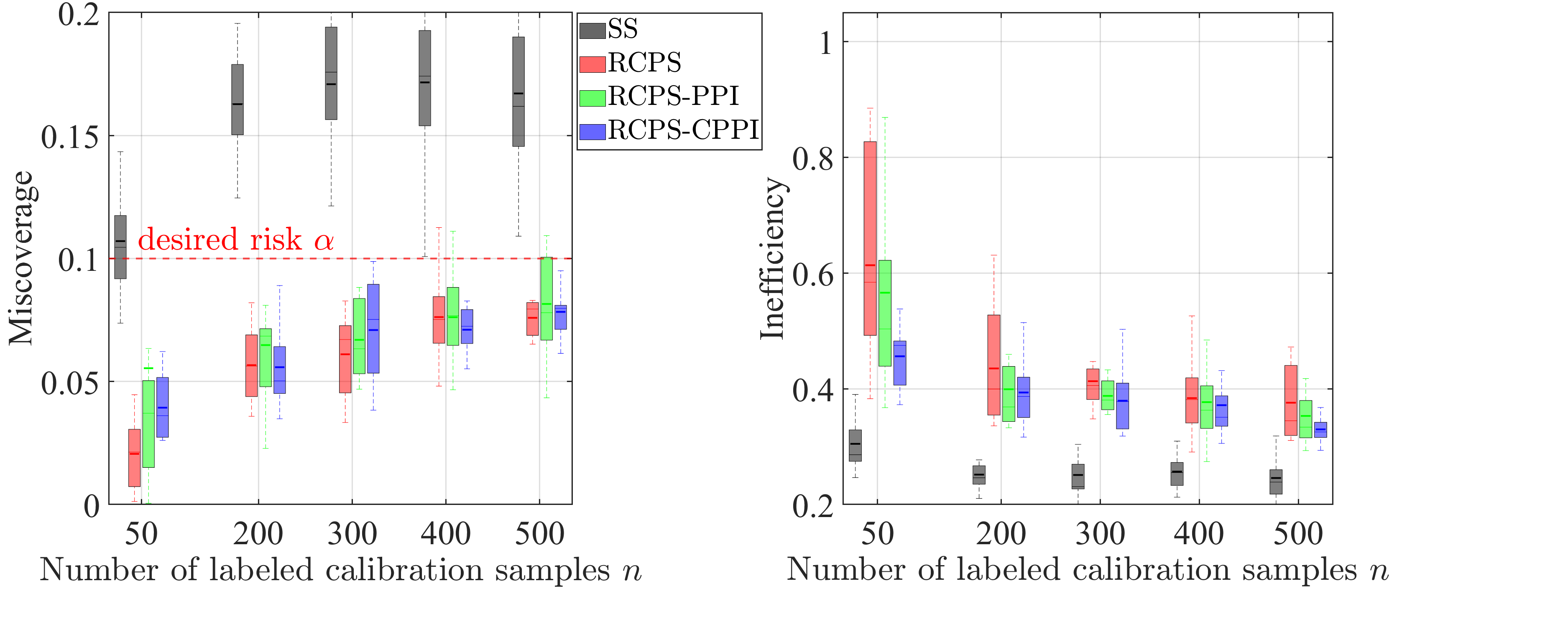}
   \caption{Empirical coverage and inefficiency of SS, RCPS, RCPS-PPI, and RCPS-CPPI versus the number of labeled calibration samples $n$ for target risk $\alpha=0.1$ and confidence $\delta=0.1$ ($N=15650$, $K=5$).}
   \label{number_of_labeled}
\end{figure}
Fig.~\ref{number_of_labeled} shows the coverage and inefficiency versus the number of labeled calibration samples $n$. All methods maintain coverage at or above the $90\%$ target for the range of $n$ tested, but their set sizes differ considerably. With very few labeled samples, RCPS produces large prediction sets to meet the risk requirement, while incorporating unlabeled data can reduce the set size. For example, at $n=50$ labeled samples, RCPS-CPPI’s sets are about 30\% smaller than those of RCPS. 
\begin{figure}[t]
   \hspace{2pt}
   \includegraphics[width=1.18\linewidth]{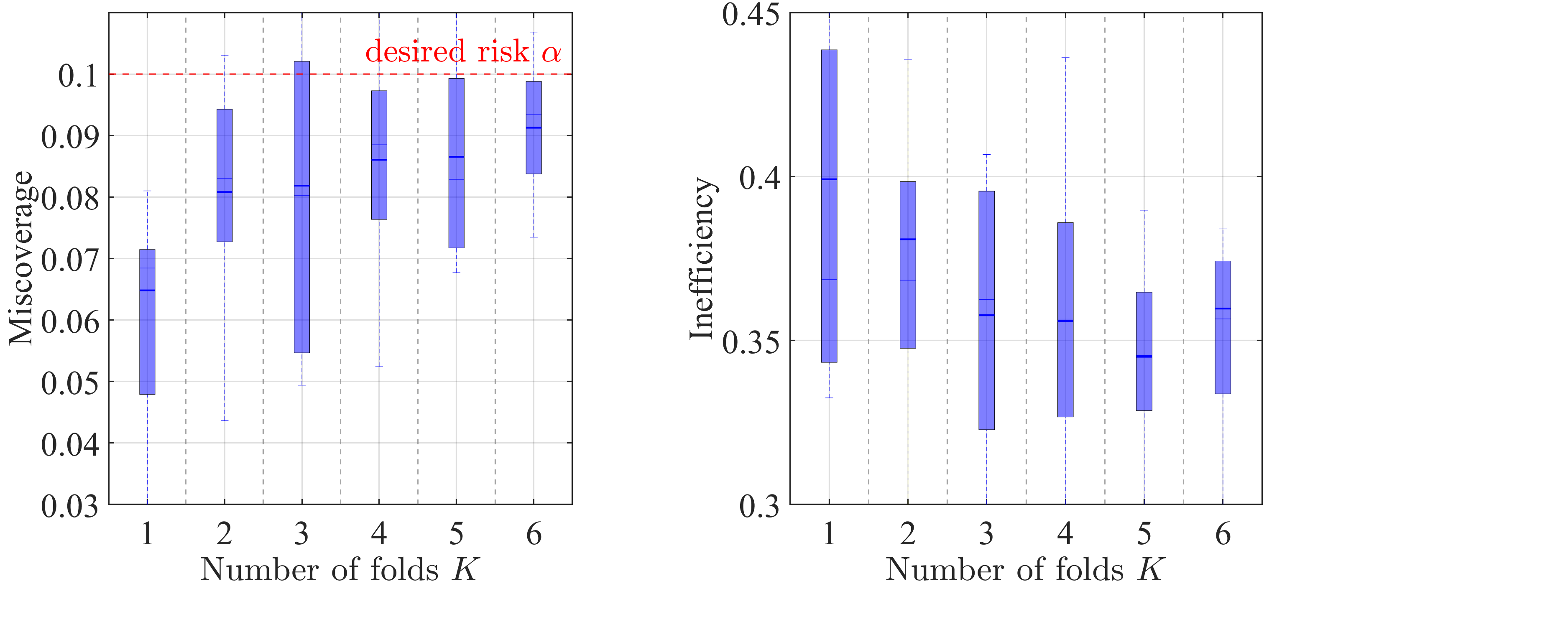}
   \caption{Empirical coverage and inefficiency of RCPS-CPPI as a function of the number of folds $K$, for $\alpha=0.1$ and $\delta=0.1$ ($N=15650$, $K=5$).}
   \label{number_of_folds}
\end{figure}

In Fig.~\ref{number_of_folds}, we examine the effect of the number of folds, $K$, on the performance of RCPS-CPPI. The total number of labeled and unlabeled calibration samples is fixed. We observe that RCPS-CPPI maintains valid coverage around the 90\% level for all values of $K$. Furthermore, the inefficiency tends to decrease as $K$ increases, since using more folds supports training the prediction model on a larger portion of the labeled data. The marginal gain from increasing $K$ diminishes once each model uses most of the data, e.g., beyond $K=5$ or $10$ in our experiments. Importantly, even for moderate values like $K=5$, RCPS-CPPI already provides a substantial improvement over the case $K=1$, which corresponds to RCPS-PPI. In practice, one can choose the number of folds, $K$, in a range that balances computational overhead with the benefits of increased training data per fold. Our results suggest that a small $K$ (e.g., 5) is often sufficient.

\begin{figure}[t]
   \hspace*{-5pt}
   \includegraphics[width=1.17\linewidth]{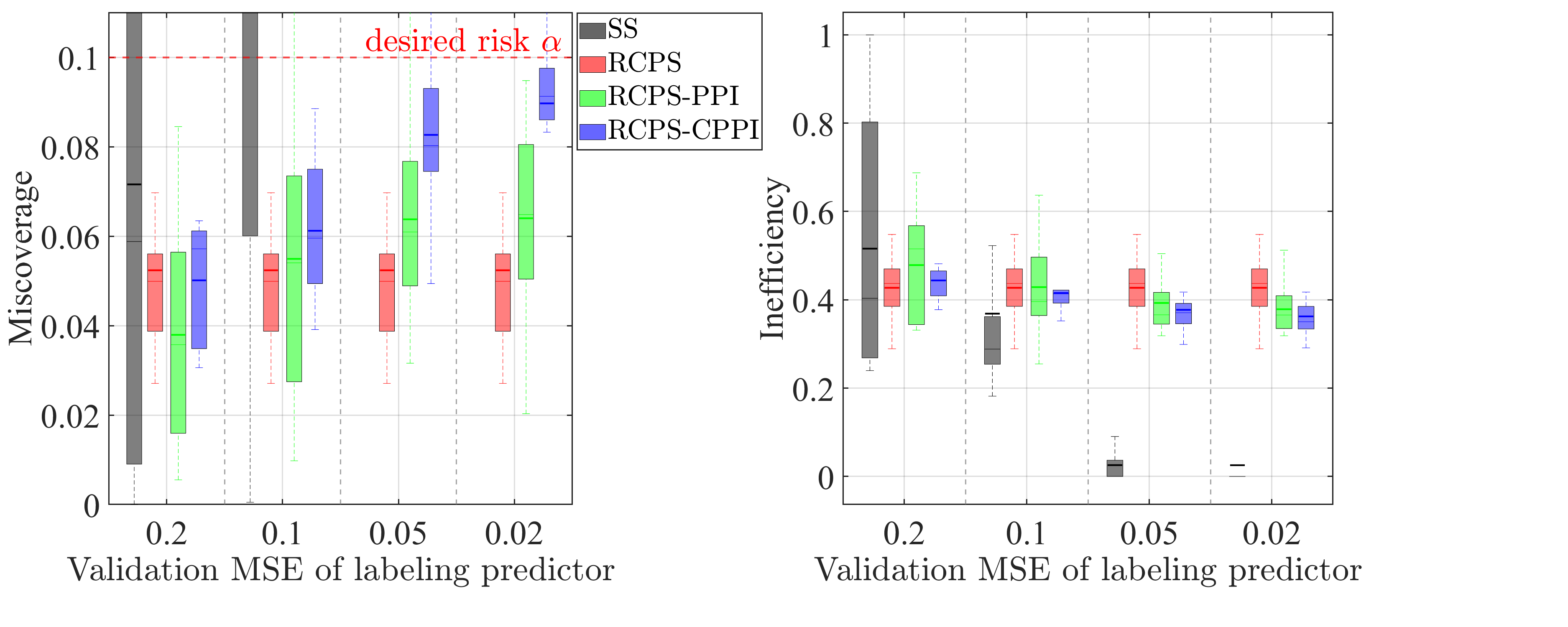}
   \caption{Empirical coverage and inefficiency of SS, RCPS, RCPS-PPI, and RCPS-CPPI versus the validation MSE of the labeling predictor $(N=15650, n=200, K=5)$.}
   \label{mse}
\end{figure}
Finally, Fig.~\ref{mse} illustrates the impact of the labeling predictor's accuracy on calibration performance. We plot the coverage and inefficiency of each method versus the mean squared error (MSE) of the predictor, measured on a validation set, in predicting $Y$. We vary the predictor’s accuracy by training with different amounts of data. The results show that the SS method, which blindly trusts the predictor, starts to exhibit under-coverage, dropping below the 90\% line, because the pseudo-labels are often incorrect. RCPS-PPI is more robust due to bias correction, but its coverage can still falter for high MSE values. In contrast, RCPS-CPPI maintains coverage near the target across the entire range of predictor qualities. In the worst case where the predictor is uninformative, RCPS-CPPI’s procedure essentially falls back to the conventional RCPS using the labeled set, thus ensuring valid risk control.

\section{Conclusion}
Wireless indoor localization increasingly relies on predictive models for accurate position estimation, making reliable uncertainty quantification essential for trustworthy and risk-aware operation. We presented RCPS-CPPI, a cross-validation-based semi-supervised calibration method that improves the sample efficiency of risk-controlling prediction sets. Leveraging $K$-fold cross-prediction, the method fine-tunes a predictor on all available labeled data while obtaining unbiased estimates from unlabeled data. We derived a rigorous confidence bound for the CPPI risk estimator. Experiments on a fingerprinting dataset demonstrated that RCPS-CPPI achieves target coverage with significantly smaller prediction sets compared to conventional RCPS and RCPS-PPI. Further directions for future research may include extending RCPS-CPPI to support online and adaptive calibration for time-varying wireless channels, and enabling real-time reliability assurance in dynamic indoor localization environments.

\bibliographystyle{IEEEtran}
\bibliography{ref}

\end{document}